\documentclass[]{article}

\usepackage{arxiv}

\usepackage{mathtools}  % Loads and improves amsmath
\usepackage{amssymb}    % Loads amsfonts automatically
\usepackage{amsthm}     % For arxiv

\newtheorem{theorem}{Theorem}

\newtheorem{definition}{Definition}

\usepackage{mathtools}  % Loads and improves amsmath
\usepackage{amssymb}    % Loads amsfonts automatically
\usepackage{dsfont}     % For \mathds{R} etc.

\usepackage[T1]{fontenc}
\usepackage{xspace}
\usepackage{float}      % Allows [H] placement for figures/tables
\usepackage{hyperref}

\usepackage{graphicx}
\usepackage{adjustbox}
\usepackage{booktabs}   % For \toprule, \hline, \bottomrule
\usepackage{multirow}

\usepackage{algorithm}   % The "wrapper" (provides the floating environment)
\usepackage{algorithmic} % The "content" (provides the \STATE, \IF, etc.)

\usepackage{xcolor}

\newcommand{\ignore}[1]{}

\DeclarePairedDelimiter\floor{\lfloor}{\rfloor}

\newcommand{\gsemo}[1]{GSEMO\xspace}

\begin{document}

\author{
Liam Wigney\\
Optimisation and Logistics\\
School of Computer Science \\ and Information Technology\\
Adelaide University\\
Adelaide, Australia
\And
Frank Neumann\\
Optimisation and Logistics\\
School of Computer Science \\ and Information Technology\\
Adelaide University\\
Adelaide, Australia
}

\title{Multitask Pareto Optimization for Monotone Submodular Problems with Dynamic Constraints}
\maketitle

\begin{abstract}
Evolutionary multitasking is a recent approach that solves multiple related optimization problems within a single evolutionary run, rather than addressing each problem separately. We consider monotone submodular optimization problems with dynamic knapsack constraints and study a multitasking formulation in which all tasks share a common monotone submodular function $f$, but differ in their constraints.
We focus on the case where elements within each constraint have uniform cost and show that this structure leads to small Pareto fronts in the multitasking formulation. This enables solution sharing across tasks and can improve performance compared to running standard evolutionary approaches independently, depending on the constraint regime.
Using rigorous runtime analysis, we analyze the expected time until the proposed multitasking algorithms obtain a $(1 - 1/e)$-approximation for each task. Experimental results for the Maximum Coverage problem complement the theoretical analysis and provide further insight into the practical behavior of the approach across different budget settings.
\end{abstract}

\section{Introduction}

Evolutionary multitasking is an approach that allows multiple related problems to be solved using a single population of an evolutionary algorithm, rather than separately with their own populations as done traditionally ~\cite{DBLP:journals/cogcom/OsabaSMH22}. It has been shown to effectively solve problems in many domains, such as combinatorial multi-objective optimization, optimization as a service and genetic programming ~\cite{DBLP:conf/soict/Ong15,Tan2023,DBLP:conf/cec/DaGOF16}.

Real-world problems that can be modeled using diminishing returns and theoretical problems such as max coverage ~\cite{DBLP:journals/ipl/KhullerMN99} and max cut ~\cite{DBLP:journals/jacm/GoemansW95} are examples of combinatorial optimization problems called submodular problems. These functions can be seen as the discrete version of the continuous concept of convexity. Constrained variations of these problems have been shown to be effectively solved using evolutionary algorithms ~\cite{DBLP:books/cu/p/0001G14,DBLP:journals/mor/NemhauserW78}. It is NP-hard to maximize these constrained problems, even for trivial constraints on the number of items.

For simple EAs such as $(1+1)$-EA, these problems typically get trapped in local optima, although archive-based methods have shown improvements ~\cite{DBLP:conf/ppsn/NeumannR24}. These problems can be formulated in a multi-objective way, which allows for them to be solved using GSEMO (global simple evolutionary multi-objective optimizer), a multi-objective evolutionary optimization algorithm.

In practice it is common that these constraints are dynamic, often they increase and decrease with time. For example, logistical limitations might mean that during the transport of goods, the number of available trucks fluctuates over the day. Optimization algorithms then need to be able to ensure that they can respond effectively to these changes to make sure that the optimal goods are still prioritized ~\cite{DBLP:journals/ai/RoostapourNNF22,DBLP:conf/ijcai/0001Q0021}.

We consider a multitasking formulation of GSEMO in which the knapsack bounds may vary dynamically over time. An alternative perspective is to view each change in the bound as inducing a new static problem that is initialized with the population obtained for the previous bound. This perspective makes it possible to examine how Pareto-based selection in GSEMO retains, discards, and transfers solutions as the constraint changes over time.

We investigate this setting using a combination of rigorous runtime analysis and extensive experimental studies with statistical evaluation. The runtime analysis provides insight into the asymptotic performance of the algorithm, while the experiments are used to explore how multitasking behaves in practice. Further details on the theoretical methodology can be found in \cite{DBLP:series/ncs/2020DN,DBLP:conf/gecco/Witt14a}.

In this work, we study evolutionary multitasking for monotone submodular optimization with dynamic knapsack constraints, with particular emphasis on the induced Pareto-front structure and the resulting solution-sharing mechanisms observed in practice.

\subsection{Related Work}

Most existing work on evolutionary multitasking has focused on continuous optimization problems \cite{DBLP:journals/tcyb/LinLTG21,DBLP:journals/tcyb/LinWMGLC24,SUN2023504,DBLP:journals/isci/HuLSM22}. These approaches often rely on sophisticated algorithms, including neural networks and deep learning. While they achieve strong performance, their complexity typically makes a detailed analysis of the underlying dynamics difficult, and a clear understanding of the mechanisms behind them is often lacking. In comparison, evolutionary multitasking for combinatorial optimization problems is less well studied. Existing work has mainly focused on classical problems such as the traveling salesperson problem or variants of the knapsack problem \cite{DBLP:conf/gecco/WigneyNON25,DBLP:conf/gecco/Don0025,7848632}.

Theoretical results on evolutionary multitasking are also limited. Most existing analyses consider knowledge transfer mechanisms \cite{DBLP:conf/cec/ScottJ24,DBLP:conf/foga/ScottJ23} or benchmark problems \cite{DBLP:conf/gecco/Lengler0N25}. Until recently, there have been no evolutionary multitasking studies explicitly considering monotone submodular objective functions \cite{DBLP:conf/gecco/ANNON}.  As a result, the literature has so far focused primarily on performance improvements, with less emphasis on understanding how multitasking fundamentally alters the search dynamics when compared to classical single-task evolutionary algorithms.

Dynamic evolutionary multitasking has received some attention in recent years. Evolutionary multitasking for dynamic scheduling was studied in \cite{DBLP:journals/tec/ZhangMNZT21}, and multitasking under dynamic constraints has been considered in \cite{DBLP:journals/tec/WangXJT25,DBLP:journals/tec/QiaoYQ0SYLT23}. However, these works do not explicitly consider submodular or monotone objective functions. In particular, evolutionary multitasking for monotone submodular optimization under dynamic knapsack constraints has not yet been studied.

Previous work on runtime analysis for GSEMO/POMC (Pareto Optimization for maximizing a Monotone function with a monotone Cost constraint) and simple single-objective evolutionary algorithms such as $(1+\lambda)$-EA and $(1+1)$-EA have typically been done by considering single unweighted uniform constraints and general cost constraints ~\cite{DBLP:journals/ai/RoostapourNNF22,DBLP:conf/ppsn/NeumannR24}. There is also a growing body of work that considers matroid constraints but only in the static case ~\cite{DBLP:journals/ec/FriedrichN15,DBLP:conf/ppsn/DoN20}. Works that include arbitrary weights often consider chance constraints rather than deterministic uniform constraints ~\cite{NeumannNeumannTCS23,DBLP:conf/ppsn/YanNN24}. The more complex constraint types often require sophisticated methods that are not needed in the simple uniform constraint cases; however, they all give reasonable upper bounds on the runtime to reach a $\left(1-\frac{1}{e}\right)$-approximation.

Traditionally, these constraints are considered to be static throughout the runtime of the algorithm, however literature around dynamic constraints is sparse. Dynamic bounds with approximated costs and approximately submodular functions were analyzed for GSEMO ~\cite{DBLP:journals/ai/RoostapourNNF22}. Works improving GSEMO have been considered for dynamic cost constraints ~\cite{DBLP:conf/ijcai/0001Q0021}, with a focus on more efficient approximation guarantees via the new algorithm FPOMC. There also exists non-evolutionary literature on dynamic constraints ~\cite{DBLP:conf/stoc/ChenP22} but again none focus on any form of evolutionary multitasking.

Most of these algorithms are multi-objective Pareto optimization based approaches. They work by building up a Pareto front for the problem, which ensures that optimal solutions remain in the population while suboptimal solutions do not get included.
Due to this, the upper bound on population size is a significant factor that determines the runtime bounds as noted in previous papers ~\cite{DBLP:conf/ppsn/NeumannR24,DBLP:conf/ppsn/YanNN24}, as it controls the number of points and thus gives the upper bound on the Pareto front. As the Pareto front grows exponentially large as the number of trade-offs increases, for a multitasking algorithm the problems need to be correlated to achieve a benefit compared to in the classical approach where they are done separately ~\cite{Tan2023}.

\subsection{Our Contribution}

We contribute to the fundamental understanding of how multitasking evolutionary algorithms work when applied to monotone submodular problems with dynamic constraints both theoretically and experimentally. Specifically, we provide runtime bounds for GSEMO on these problems. For each problem set we find the population size and use it to derive exact upper bounds on the runtime. We prove this for the most general case with $k$ problems and then derive interesting subsets of this general case. It is because of the correlation between the constraints, as well as the nature of the problems, that we can efficiently do so.

The analysis demonstrates the multitasking approach can have a superior runtime in some scenarios compared to the classical approach by a factor of $k$. This is because the Pareto approach means that we can share the population and by finding the largest of the problems, we get all the lower order approximations for free. As the primary driver for the bounds is the size of the population, problems that are subsets of the most general problem, by sharing a bound for example, are trivial to derive. We showed that there is a relationship between the problem similarity and the possible advantage for the upper bound of the multitasking approach vs the classical approach.

The experimental analysis was done applying the NP-hard Maximum Coverage problem to multiple social graphs with different constraints and dynamic schedules. We consider multiple problems and maximize them all classically as well as using our multitasking regime. We considered the case where the constraints are dynamically being pushed to a higher bound as well as where the dynamic changes are being uniformly sampled. We observe from this the benefits and drawbacks of the multitasking approach in its ability to dynamically adapt to changes in constraints and gain insight into the mechanics behind the algorithm in practice. We saw that when the problems are similar, MT-GSEMO provides an advantage. But, as the similarity decreases, the advantage turns to a disadvantage.

This indicates that in practice, MT-GSEMO's effectiveness for a given dynamic constrained problem is dependent on the similarity between the largest dynamic constrained problem and the one being considered which scales and aligns with the results in the theory section.

The paper is structured as follows:
Section 2 outlines the preliminary definitions and notation. Section 3 gives a rigorous runtime analysis, where we begin with the most general $k$ constraint case and discuss the applicability to other scenarios and conclude with comparisons between our evolutionary multitasking approach and the classical approach. Section 4 contains experimental results and analysis.

\section{Preliminaries}

We call a function $f: \{0,1 \}^n \to \mathbb{R^+}$ the objective function, specifically formulated in terms of bit strings.
\begin{definition}
    This function is considered submodular if $\forall \ x,y \in \{0,1\}^n$ where $x \leq y$ and  $\forall i$ where $x_i=y_i=0$ the following holds
    \begin{equation*}
        f(x \oplus e_i)-f(x) \geq f(y \oplus e_i) - f(y)
    \end{equation*}
    where setting the $i$th bit of $z$ to $1$ is denoted by $z \oplus e_i$ and where $x \le y$ for $x,y \in \{0,1\}^n \implies x_i \le y_i \ \forall i \in \{1,\dots,n\}$.
\end{definition}

These problems are monotone if $f(x) \leq f(y)$ holds when $x \leq y$ for all $x,y \in \{0,1\}^n$. These functions can be seen as the discrete version of the continuous concept of convexity.

\begin{definition}
    Let $x^i \in \{0,1\}^n$ be a solution containing $i$ selected elements.
    The largest marginal gain obtainable by adding one further element to
    $x^i$ is
    \begin{equation*}
        \delta_{i+1}
        =
        \max_{\substack{j \in \{1,\dots,n\}\\ (x^i)_j=0}}
        \left(
            f(x^i \oplus e_j)-f(x^i)
        \right).
    \end{equation*}
\end{definition}

We consider uniform cost constraints $c: \{0,1\}^n \to \mathbb{R}$ that count the number of $1$-bits in a solution, $c(x) = \sum_{i=1}^n x_i = |x|_1$. It is called a linear weighted uniform constraint if it has a weight applied to this count, $c(x) = a \cdot \sum_{i=1}^n x_i = a \cdot |x|_1$ where $a \in \mathbb{R}^+$. This weight $a$ can be normalized into the bound $B \in \mathbb{R}$ such that $c^*(x) = |x|_1$ and $B' = B/a$ if needed. We keep them non-normalized to keep the results more practically useful and simplify further extensions looking into different constraints and costs. We will only be considering uniform cost constraints and for simplicity use the term constraint to refer to the cost constraint function and bound together.

\begin{definition}
    (General problem of a single objective with one constraint). Given a monotone submodular objective function $f$, cost constraint function $c$ and budget $B$,
    \begin{equation*}
                        \underset{x \in \{0,1\}^n}{\max} f(x) \text{ such that } c(x) \leq B
    \end{equation*}
\end{definition}

Even with these trivial constraints, solving these problems is NP-hard and is classically done by solving independently each problem. Let $x^{OPT}_i$ be an optimal feasible solution for problem $(f, c_i, B_i)$. Our goal is to compute for each $i \in \{1, \ldots, k\}$ a solution $x$ with $f(x) \geq \alpha \ f(x^{OPT}_i)$ and $c_i(x) \leq B_i$. The multitasking formulation however considers solving these problems simultaneously.

\begin{definition}
    (General problem of a single objective with multiple constraints). Given a monotone submodular objective function $f$, $k$ linear weighted uniform cost constraint functions $c_i(x) = a_i \cdot |x|_1$ with weights $a_i \in \mathbb{R}^+$ and $k$ budgets $B_i$, the goal is to find a set of $k$ solutions $\{x_1, \dots, x_k\}$ such that for each $i \in \{1, \dots, k\}$:
    \begin{equation*}
       x_i \in \operatorname*{arg\,max}_{x \in \{0,1\}^n:\,c_i(x)\le B_i} f(x).
    \end{equation*}
\end{definition}

These are represented as the problems defined by the triplet $(f, c_i, B_i)$ where $i \in \{1, \dots, k\}$ indexes each problem.

\begin{definition}
    The general multi-objective objective function GSEMO attempts to solve is
    \begin{equation*}
        g(x) = (g_1(x), g_2(x), \dots, g_{k+1}(x))
    \end{equation*}
    Where $g_1(x)$ is the primary submodular monotone function and $\{g_2(x), \dots, g_{k+1}(x)\}$ is the set of negative cost functions to be maximized.
\end{definition}

\begin{definition}
\label{pri:multitask}
        The most general multitasking primary objective function shared by each $i \in \{1, \dots, k\}$ problems is defined as:
        \begin{equation*}
                g_1(x) =
                \begin{cases}
                        f(x), & \text{if }\ \exists i \in \{1, \ldots, k\} \text{ with} \ c_i(x) \leq B_i\\
                        -1,   & \text{else}
                \end{cases}
        \end{equation*}
\end{definition}

In our multitasking formulation, we consider multiple problems where for $k$ problems we have $k$ secondary functions defined as the negative of the individual cost functions, such that $g_{i+1}(x) = -c_i(x)$ for $i \in \{1, \dots, k\}$. This means that overall we have $k+1$ different objective functions, where the primary objective function, the submodular monotone function, is fully shared between the problems and only needs to be evaluated once per generation. Notably, this ensures that there is no negative transfer as is possible under different multitasking formulations.

Maximum Coverage is defined on an undirected graph $G = (V,E)$ where $n=|V|$. Each node $v$ has a cost of $1$ and $N(v)$ is the set of nodes containing both $v$ and its neighbors. Let $x$ be the set of nodes selected, then the coverage of this set is
\begin{equation}
\label{eqn:coverage}
    \text{Coverage}(x) = \left| \bigcup_{i: x_i=1} N(v_i) \right|
\end{equation}

By using a multi-objective algorithm, we can set the secondary objectives to maximize negative the number of $1$'s considering the weights of the given cost. This allows us to maximize $f$ while minimizing $c$, while also only including feasible solutions. We do this by introducing the concept of Pareto domination and a new objective function.

\begin{definition} (Pareto domination).
        Let $g=(g_1, \ldots, g_k) \colon \{0,1\}^n \rightarrow \mathds{R}^k$ be an objective function with $k$ objectives that should be maximized.
        Given two search points $x, y \in \{0,1\}^n$, we say that $x$ dominates $y$ ($x \succeq y$) iff $g_i(x) \geq g_i(y)$, $1 \leq i \leq k.$ We say that $x$ strongly dominates $y$ ($x \succ y$) iff $x \succeq y$ and $g_i(x)>g_i(y)$ for at least one $i \in \{1, \ldots, k\}$.
\end{definition}

When the constraint varies, the algorithm immediately re-evaluates the population. This does not immediately kick out infeasible solutions, rather the next domination check does. The general method for each proof is to first setup the problem by defining the constraints and the bounds to give us $B_{max}$, the number of possible $1$-bits a feasible solution can have also referred to as the maximum bound. We then derive an upper bound on the population size $U$ dependent on $B_{max}$ and $n$.

Then we can find the time needed until the first Pareto optimal individual $0^n$, the individual with only $0$ bits, is in the population. Finally, we use induction to find the time until the rest of the Pareto front is found. We measure runtime as the number of evaluations of the function $f$. The expected time then is the expected number of fitness evaluations of the function $f$. We primarily give the results in terms of $U$ rather than $B_{max}$ as it makes for a more natural representation of that the upper bound is and in our case they are asymptotically the same. If the constraints were more costly, then we would need to take them into account, but we assume that the function evaluation is the most expensive operation. Then for large enough generations, the cost of the comparisons is comparably small.

\begin{algorithm}[tb]
    \scriptsize
    \raggedright
    \caption{GSEMO}
    \label{alg:gsemo}
    \begin{algorithmic}[1]
        \STATE Choose $x \in \{0,1\}^n$ uniformly at random;
        \STATE $P \gets \{x\}$;
        \STATE $t \gets 0$;
        \REPEAT
            \IF{constraints change}
                \STATE Re-evaluate $g(z)$ for all $z \in P$;
            \ENDIF
            \STATE Choose $x \in P$ uniformly at random;
\STATE Create $y$ by flipping each bit of $x$ independently with probability $1/n$;
\IF{$\nexists w \in P : w \succ y$}
                \STATE $P \gets (P \setminus \{z \in P \mid y \succeq z\}) \cup \{y\}$;
            \ENDIF
            \STATE $t \gets t + 1$;
        \UNTIL{$t \geq t_{max}$}
        \RETURN $P$;
    \end{algorithmic}
\end{algorithm}

\section{Dynamic constraints}

In this section we consider what happens when we are optimizing the previously discussed problems, and the constraints of one or more of these problems vary. We consider how their maximum effective bound $B_{max}$ and population size upper bound $U$ changes, to $B_{max}^*$ and $U^*$ respectively. We investigate what happens as the maximum effective bound remains the same, decreases and increases. While each new constraint does give rise to a new set of static constraints allowing the prior sections results to be used, these results give tighter bounds in terms of the upper bound required to do the update. The maximum effective bound refers to the maximum feasible number of elements, $B_{max} = \min \left( \max_{1 \le i \le k}\left( \floor*{\frac{B_i}{a_i}} \right), n \right)$, which determines the upper bound on the Pareto front size and hence the population upper bound $U$.

We look only at a single update of constraints at a time, thus it is not possible for example that $B_{max}^*$ becomes $0$ and then returns to $B_{max}$ (which would take two or more updates). When the constraint varies, the algorithm immediately re-evaluates the population. This does not immediately kick out infeasible solutions, rather the next domination check does. Note that if we have a population that contains $\left( 1-\frac{1}{e} \right)$-approximations for each problem, then as shown in the proof of Theorem \ref{pro:proof}, we actually have these approximations for every single $b \in \{0, \dots, B_{max}\}$. We assume that the population currently contains these.

\subsection{Maximum effective bound unchanged}

An increase or decrease in bounds and weights that does not change the maximum effective bound $B_{max}$, and thus the population size upper bound, does not require additional function evaluations. As noted previously, we only look at the situation where the constraints update once and so it is not possible for the algorithm to decrease from $B_{max}$ to $B_{max}^*$, remove the now infeasible population that included $B_{max}$ and then increase back to $B_{max}$. If this were possible, we would require more evaluations. It was shown in ~\cite{DBLP:conf/gecco/ANNON}, and can also be inferred from the argument in Theorem ~\ref{pro:proof}, that the upper-bound for finding approximations for each problem is $\mathcal{O}(Un(\log n + U))$ and that in finding $\left( 1-\frac{1}{e} \right)$-approximations for every $u\in \{0, \dots, B_{max}\}$, the new solution is already contained in the population.

\subsection{Decrease in maximum effective bound}

A decrease in bounds or increase in weights such that the maximum effective bound and thus population size upper bound decreases, does not require additional function evaluations. When bounds decrease the number of possible bits clearly decreases. When the weights increase, each bit costs more and so the number of possible bits also decreases. These changes only decrease $B_{max}$ when they occur to the problem that maximized $B_{max}$, however this problem $h$ is not necessarily the problem that maximizes $B_{max}^*$, as another problem might now be maximizing $B_{max}^*$. As explained in the stable bound change section, when we decrease $B_{max}$ to $B_{max}^*$ and thus $U$ to $U^*$, the $\left( 1-\frac{1}{e} \right)$ approximation for these individuals is already in the population.

\subsection{Increase in maximum effective bound}

An increase in the bounds or decrease in the weights such that maximum effective bound and thus the population size upper bound increases to $B_{max}^*$ and $U^*$ requires a proof, as we need to find the expected time it takes to find the new approximations that are not already contained in the population. This also requires finding the new upper bound on the population size noting that the new maximum effective bound is not necessarily the same problem that maximized $B_{max}$.

\begin{theorem}
        \label{new_pop}
        Given a change in budgets and weights that increases the maximum effective bound from $B_{max}$ to $B_{max}^*$, the population size is upper bounded by $U^*=\min(\max_{1 \le i \le k}\left(\floor*{\frac{B_i^*}{a_i^*}}\right), n) + 1$.
\end{theorem}

\begin{proof}
    We have new constraints of the form $a_i^* |x|_1 \leq B_i^*$, which implies that the maximum number of bits, $|x|_1$,  a feasible solution can have is $ \max_{1 \le i \le k}\left(\floor*{\frac{B_i^*}{a_i^*}}\right)$. This can potentially be larger than $n$ and so the actual possible maximum value of $|x|_1$ is $B_{max}^*$ which is the minimum of either $n$ or $\max_{1 \le i \le k}\left(\floor*{\frac{B_i^*}{a_i^*}}\right)$. This is analogous for the old set of constraints.

    Let $x$ be the parent and $y$ be the offspring with the same number of elements, $|x|_1 = |y|_1$. All of the secondary objectives are equal as they have the same number of bits, $-a_i |x|_1=-a_i |y|_1$. If the primary objective of the offspring is larger than the parent, $g_1(y) \geq g_1(x)$, then $y$ replaces $x$. Otherwise the offspring $y$ is discarded.

    Thus due to this Pareto domination, there is at most one individual for each $i \in \{0, \ldots, B_{max}^*\}$ and the upper bound on the population size is
    \begin{equation*}
        U^*=\min(\max_{1 \le i \le k}\left(\floor*{\frac{B_i^*}{a_i^*}}\right), n) + 1
    \end{equation*}
\end{proof}

\begin{theorem}
\label{pro:proof}

         Given we have a population $P$ such that for each of the $k$ monotone submodular problems $(f, c_i, B_i)$, we have a $\left(1-\frac{1}{e}\right)$-approximation. After changing the problems such that $U^* > U$, the expected time until GSEMO has obtained for each monotone submodular problem $(f, c_i^*, B_i^*)$, $i \in \{1, \dots, k\}$, a $\left(1-\frac{1}{e}\right)$-approximation is $\mathcal{O}(U^*n (|U^*-U|))$.
\end{theorem}

\begin{proof}
        We start from a population that includes an approximation for $B_{max}$ and with $x^{OPT}_{b^*}$ as an optimal solution for $b^*$ where $b^* \in (B_{max}, \dots, B_{max}^* ]$. We need to show that for every $b^*$ and $i \in \{B_{max}, \dots, b^*\}$ that

\begin{equation}
                \label{eqn:assum}
f(x^i) \geq \left(1-\left(1-\frac{1}{b^*}\right)^{i}\right) \cdot f(x^{OPT}_{b^*}) \quad
\end{equation}

    holds and find a $\left(1-\frac{1}{e}\right)$-approximation.

        When $i=B_{max}$, the inequality holds by assumption as we already have approximations for every possible $b \in \{0, \dots, B_{max}\}$ in the current population. Assume that this equation holds for a given $x^i$. We need to find a bound on the optimal solution and by using submodularity along with noting that an optimal feasible solution $x^{OPT}_{b^*}$ contains at most $b^*$ 1-bits, the function's monotonicity and the definition of the largest marginal increase
        \begin{equation*}
                \begin{split}
                        f(x^{OPT}_{b^*}) & \leq f(x^i \lor x^{OPT}_{b^*}) \quad \\
                        & \leq f(x^i)+\sum_{j: (x^{OPT}_{b^*})_j=1 \land (x^i)_j=0} (f(x^i \oplus e_j)-f(x^i)) \\
                        & \leq f(x^i) + |x^{OPT}_{b^*}|_1 \delta_{i+1} \leq f(x^i) + b^* \delta_{i+1}
                \end{split}
        \end{equation*}
        Rearranging this gives that the next marginal gain is bounded below by
    $\delta_{i+1} \geq \frac{1}{b^*} (f(x^{OPT}_{b^*})-f(x^i))$.

        Again, using the monotonic and submodular nature of $f$
        \begin{equation*}
                f(x^{i+1}) \geq f(x^i) + \delta_{i+1} \geq f(x^i) + \frac{1}{b^*} (f(x^{OPT}_{b^*})-f(x^i)) \quad
        \end{equation*}

        Considering the assumptions about $f(x^i)$ from Equation \ref{eqn:assum} we get
        \begin{equation*}
                \begin{split}
                        f(x^{i+1}) & \geq f(x^i)\left(1-\frac{1}{b^*}\right)+\frac{1}{b^*}\cdot f(x^{OPT}_{b^*})\\
                        %& \geq \left(1-\left(1-\frac{1}{b^*}\right)^i\right)\cdot f(x^{OPT}_{b^*}) \left(1-\frac{1}{b^*}\right)+ \frac{1}{b^*} \cdot f(x^{OPT}_{b^*}) \\
            & \geq \left(1-\left(1-\frac{1}{b^*}\right)^{i+1} \right) \cdot f(x^{OPT}_{b^*}) \\
                \end{split}
        \end{equation*}

        When the final solution $i=b^*$ is reached, we get
        \begin{equation*}
                f(x^{b^*}) \geq \left(1-\left(1-\frac{1}{b^*}\right)^{b^*}\right) \cdot f(x^{OPT}_{b^*}) \geq \left(1-\frac{1}{e}\right) \cdot f(x^{OPT}_{b^*}) \quad
        \end{equation*}

        This holds for any $b^*$. Applying this to all $b^*$ up to $B_{max}^*$ will account for all of the possible new constraint values for each monotone submodular problem as every possible value is found. When $b^* = B_{max}^*$, we have the required $\left(1-\frac{1}{e}\right)$-approximation and all others are found.
        The time to get to the final $\left(1-\frac{1}{e}\right)$-approximation for $B_{max}^*$ is dependent on the additional steps required, as the population includes a $\left(1-\frac{1}{e}\right)$-approximation for $B_{max}$ we only need to do this for each $b^* \in (B_{max}, \dots, B_{max}^* ]$. For each step from $i$ to $i+1$, we have the probability of selecting the best individual as at least $1/U^*$ and the probability of flipping the right bit as at least $1/(en)$ as we only need the element with the largest marginal increase to be flipped. We repeat each step at most $d=|U^*-U|$ times so we get the runtime upper bound to be $\sum^{d}_{i=0} (\frac{1}{U^*en})^{-1} = \mathcal{O}(n \ U^* \ d) = \mathcal{O}(n \ U^* \ |U^*-U|)$.
\end{proof}

\subsection{Subsets of the general case}

The general case can be reduced to more specific scenarios as again the effective maximum bound on the number of possible bits that can be chosen $B_{max}$ is the key component of the proofs. In every case, decreasing $B_{max}$ does not require more time to get approximations as they already exist in the population as shown previously. However, increases in $B_{max}$ does, as approximations must be found for the new possible values.

The concept of dynamic constraints can be used to analyze the case where we move from one problem to several problems. Using the same logic above we only need to find the upper bound on the population size for the old single constraint and for the new maximum effective bound, before we are able to derive the expected time.

\subsection{Comparison between methods}

If we have $k$ problems, the expected runtime to find a $(1 - \frac{1}{e})$-approximation for every problem in the multitasking scenario is
$
T_{m} = \mathcal{O}(U^*n (|U^*-U|)),
$
where $U = \max_{1 \le i \le k}(U_i)$ is the prior largest upper bound on the population across all problems and $U^* = \max_{1 \le i \le k}(U_i^*)$ is the current largest upper bound on the population across all problems. In the classical approach, each problem is solved independently, giving $k$ separate runtimes:
$
T_c = \mathcal{O}\left( \sum_{i=1}^k U_i^* n (|U_i^*-U_i|) \right).
$

We only have upper bounds, which limits our comparison. But, since for any $i$, $U^* \geq U_i^*$, the multitasking approach is never going to be asymptotically worse than the classical approach in terms of upper bounds for solving all of the $k$ problems. The speed up is entirely dependent on the similarity between $U_i^*$'s. The best speed up comes when $U_1^* \approx U_2^* \approx \dots \approx U_k^*$, as in this case the sum in $T_c$ will approach $k$ times the upper bound on $T_m$. In the worst case however, where $U_k^* \gg \sum^{k-1}_{i=1} U_i^*$, then $T_c \approx T_m$ and there is no worst case speed up at all.

We can see that when the problems are close enough and $k$ is large, there is up to a factor of $k$ speed up for the multitasking approach. However, the further apart the problems are, both before and after the dynamic change, the more likely it is that the largest bound will dominate the runtime giving no advantage to the multitasking approach. This fits the mechanisms of the algorithm, as changes in constraints that do not change the maximum effective bound will not impact the progress of the algorithm and these new budgets solutions will thus be contained in the Pareto front. When changes increase the maximum effective bound, MT-GSEMO is then able to re-use the old Pareto front from the prior maximum effective bound.

\section{Experiments}

To experimentally investigate the multitasking approach, we consider the Maximum Coverage problem from Equation \ref{eqn:coverage}, with uniform $c(v_i)=1$ weights. To allow for a more controlled comparison, we begin with an initial solution $0^n$.  We test on the social graphs ca-GrQc, Erdos992, ca-HepPh, ca-AstroPh, ca-CondMat, taken from the Network Repository \cite{DBLP:conf/aaai/RossiA15}. We test over $100,000$, $200,000$, $500,000$, and $1,000,000$ objective-function evaluations, each run 30 times. We denote the mean and standard deviation of the partial offline error from each approach by $\mu$ and $\sigma$, with the subscript $c$ for the classical algorithm and $m$ for the multitasking algorithm.

Given $s^i$ is the best solution right before the $i$-th change happens and $s^i_b$ is the baseline solution right before the $i$-th change happens. This gives us the error for one change as $e_i = f\left(s^i_b\right)-f\left(s^i\right)$ and for all $m$ changes as
$
e=\sum_{i=1}^m \frac{e_i}{m}
$.
This is used rather than coverage directly, as it allows us to compare how well the two approaches can adapt to the dynamic changes throughout the runtime rather than just how well they optimize the final constraints. While investigating the immediate changes following a constraint change gives insight into the mechanism of the algorithm, space constraints force us to focus on the more important overall metric. The baselines were found by running GSEMO for $5$M generations as exact solvers are computationally infeasible.

We begin with an initial constraint and use it for a warm-up period of $5000$ generations. We then change the bounds to a random value in an interval dependent on the bound every $\tau = 5000$ generations until we reach the end of the algorithm. We either push this constraint towards a ceiling where it then uniformly is sampled using an interval relative to the budget itself or we simply uniformly sample the constraint from the start within an interval around the current budget. This allows us to understand how well the algorithm adapts over different magnitudes and directions of dynamic change.

While the warm-up period gives the algorithm a chance to build up a front, it is not likely to be long enough for the larger problems that we can assume it is found close to optimal solutions. Nor is the gap between the changes in budget large enough to ensure the algorithm always adapts fully to the changes. This gives us the ability to actually see how well the algorithm can cope with these changes in these non-optimal scenarios.

For the $i$-th problem, let $B_{g, i}$ denote its goal upper bound, and $B^{(t)}_i$ denote its budget at the $t$-th environmental change. The maximum step size is $\Delta_i = \max(1, \lfloor 0.15 \cdot B_{g, i} \rfloor)$. For directed push constraints, the budget begins at $B^{(0)}_i = B_{g, i-1}$ (with $B^{(0)}_1 = 0$). While $B^{(t)}_i < B_{g, i}$, the update rule is: $B^{(t+1)}_i = \max(0, \min(B_{g, i}, B^{(t)}_i + \delta^{(t)}_i))$ where $\delta^{(t)}_i \sim U(0, \Delta_i]$ with probability $0.7$, and $\delta^{(t)}_i \sim U[-\Delta_i, 0]$ with probability $0.3$. Once $B_{g, i}$ is reached, the directional bias is removed and the budget fluctuates according to the uniform random walk constraints described below. For uniform random walk constraints, the budget begins at $B^{(0)}_i = B_{g, i}$. The budget updates as:
$B^{(t+1)}_i = \max(0, \min(n, B^{(t)}_i + \delta^{(t)}_i))$
where $\delta^{(t)}_i \sim U[-\Delta_i, \Delta_i]$.

We compare approaches not in terms of generations, but in terms of objective function evaluations to take into account the reduction in primary objective function evaluations due to solving multiple problems in parallel. The multitasking cases thus have a generation bound $k$ times the classical cases. In our dynamic experimental setup, environmental changes occurred at fixed generation intervals ($\tau$) for both methods. Thus, the multitasking approach performed $k$ times fewer objective function evaluations per dynamic environment compared to the classical approach. The constraint evaluations are extremely cheap, as are the Pareto dominance checks when implemented efficiently, while the function evaluations are comparatively much more expensive which allows for this comparison. Kruskal-Wallis tests with $p \leq 0.05$ are used to check for significance. While these are rank-based tests, we report mean and standard deviations in the experiments as there was an absence of outliers.

\subsection{Results}

Table \ref{tab:combined} shows similar results regardless of dynamic regime. For the largest dynamic problems, we see the multitasking approach significantly outperform the classical approach. While for the smallest bounds, we see the classical approach almost always outperform the multitasking approach. The results for the middle-sized problems, however, differ depending on the type of dynamic regime. In the push to an upper bound problem, the multitasking approach was only better for the smallest number of evaluations while in the uniform random walk constraints it was typically better for much longer and this is consistent regardless of the density of the graphs. For most problems, we see the mean and standard deviation of the error decrease as the evaluations increase. The main time we do not see this is for the trivial bound, where we rarely see the multitasking approach be less effective as the number of evaluations increase as the initial $0$-bit string gives a small error initially, before the large search space leads the next few solutions to be poor for this problem. The second biggest dynamic constraint problems are still around $50\%$ smaller than the largest, which shows that the multitasking algorithm is still able to be effective when the problems are dissimilar. The middle problems are roughly a factor of $10$ smaller, and so it is around here that the classical approach begins to be more effective given enough evaluations.

The results for the smaller problems are likely due to the reduced search space that the classical approach has. It only has to search for new solutions within this smaller space, allowing optimal solutions for the new bounds to be found sooner than in the multitasking approach. This is true for even the trivial constraint, where the multitasking approach struggles significantly to adapt to these changes. As the evaluations increase though, the multitasking approach improves faster than the classical and the absolute difference decreases.

The results for the largest problem are due to the multitasking ability to have effectively more generations for free, however for the second largest, the improvement is likely due to the multitasking approach's ability to share the evaluation of the objective and population. If a change occurs that increases its budget below the largest problem, the new budget will still be in the population and no additional work will need to be done. This effect is larger than the disadvantage that faced the smaller problems with the expanded search space. The results for the middle constraint and the difference between the two approaches can be seen as the reuse of existing solutions being more useful for the uniform random walk constraint as the budget stays closer to the original budget. When moving to a larger budget closer to the largest problem, such as in the push experiments, being larger allows the expanded search space penalty to be lessened.

\section{Conclusions}

We considered a multitasking formulation for monotone submodular optimization under dynamically changing knapsack constraints. The multitasking approach leads to small Pareto fronts, which allows solutions to be shared between related problems as the constraints change over time. Using runtime analysis, we studied the expected time until a $\left(1-\frac{1}{e}\right)$-approximation is obtained again after a change in the constraints, giving insight into how quickly the algorithm adapts to new bounds. We showed that there is a clear relationship between the benefits of the multitasking approach with the similarity of the problems being considered.

Our experimental investigation using the Maximum Coverage problem provides further insight into how multitasking behaves in dynamic settings. We considered both uniform random walk bounds and randomly increasing bounds up to a fixed point. We observed that multitasking is most effective when the dynamic constraints are within an order of magnitude of the largest budget, where solution sharing between problems is greatest. For smaller budgets, however, the classical GSEMO performs better, which is likely due to the larger search space induced by multitasking. This aligns with the theoretical results and gives further insight by showing that in compute-restricted environments, the multitasking approach can give worse results than the classical approach.

In this paper, we focused on monotone submodular problems with uniform weights; relaxing this to more general weight structures or approximately submodular objectives would be interesting. Sliding window techniques could be used to limit the growth of the Pareto front in dynamic settings, helping to control the larger search space induced by multitasking. This may reduce the overhead observed for smaller budgets, while preserving the ability of the algorithm to retain and reuse solutions when the constraints increase. In this way, sliding windows could help balance the tradeoff between maintaining solutions for large budgets and efficiently adapting to smaller ones as the constraints change.

\newpage

\begin{table}[H]
\centering
\caption{Results for dynamic constraints. +* indicates Multitasking was significantly better, -* indicates Classical was better, = indicates no difference. Best results are highlighted in bold.}
\label{tab:combined}
\tiny
\begin{adjustbox}{max totalheight=\textheight, keepaspectratio, max width=\textwidth}
\begin{tabular}{|c|c|c||cc|cc|c||cc|cc|c|}
\hline
\multicolumn{3}{|c||}{} & \multicolumn{5}{c||}{\textbf{Directed Push Constraints}} & \multicolumn{5}{c|}{\textbf{Uniformly Random Dynamic Constraints}} \\
\hline
\multicolumn{1}{|c|}{} & \multicolumn{1}{c|}{} & \multicolumn{1}{c||}{} & \multicolumn{2}{c|}{\textbf{Classical}} & \multicolumn{2}{c|}{\textbf{Multitasking}} & \multicolumn{1}{c||}{} & \multicolumn{2}{c|}{\textbf{Classical}} & \multicolumn{2}{c|}{\textbf{Multitasking}} & \multicolumn{1}{c|}{} \\
\multicolumn{1}{|c|}{Graph} & \multicolumn{1}{c|}{Bound} & \multicolumn{1}{c||}{Eval.} & Mean & Std & Mean & Std & Stat. & Mean & Std & Mean & Std & Stat. \\
\hline
\multirow{20}{*}{Erdos992}
 & 1 & 100000 & \textbf{5.9} & \textbf{1.5} & 98.3 & 14.4 & -* & \textbf{1.9} & \textbf{0.7} & 29.4 & 7.4 & -* \\
 & 1 & 200000 & \textbf{12.9} & \textbf{2.2} & 83.1 & 12.6 & -* & \textbf{2.1} & \textbf{0.5} & 22.7 & 5.3 & -* \\
 & 1 & 500000 & \textbf{10.0} & \textbf{1.4} & 52.2 & 8.7 & -* & \textbf{4.2} & \textbf{0.6} & 17.0 & 3.6 & -* \\
 & 1 & 1000000 & \textbf{8.6} & \textbf{0.8} & 31.4 & 5.2 & -* & \textbf{4.3} & \textbf{0.3} & 14.2 & 2.7 & -* \\
 & 12 & 100000 & \textbf{57.0} & \textbf{8.9} & 234.2 & 15.6 & -* & \textbf{35.8} & \textbf{5.5} & 89.2 & 14.3 & -* \\
 & 12 & 200000 & \textbf{40.7} & \textbf{6.0} & 262.9 & 14.1 & -* & \textbf{21.8} & \textbf{3.0} & 133.3 & 15.5 & -* \\
 & 12 & 500000 & \textbf{33.6} & \textbf{3.6} & 213.6 & 12.3 & -* & \textbf{14.4} & \textbf{1.7} & 126.3 & 14.2 & -* \\
 & 12 & 1000000 & \textbf{34.9} & \textbf{3.3} & 129.5 & 8.8 & -* & \textbf{20.2} & \textbf{2.6} & 78.9 & 9.3 & -* \\
 & 78 & 100000 & 857.6 & 27.4 & \textbf{808.4} & \textbf{32.2} & +* & 829.3 & 21.0 & \textbf{703.9} & \textbf{32.7} & +* \\
 & 78 & 200000 & 668.3 & 24.7 & 670.2 & 28.0 & = & 616.0 & 20.3 & \textbf{527.5} & \textbf{25.9} & +* \\
 & 78 & 500000 & \textbf{426.7} & \textbf{22.3} & 465.6 & 19.9 & -* & 345.1 & 16.2 & \textbf{304.3} & \textbf{16.8} & +* \\
 & 78 & 1000000 & \textbf{268.2} & \textbf{15.2} & 277.4 & 12.0 & -* & 191.1 & 9.5 & \textbf{171.1} & \textbf{10.4} & +* \\
 & 305 & 100000 & 2285.0 & 41.1 & \textbf{1449.5} & \textbf{33.8} & +* & 2278.7 & 38.5 & \textbf{1439.9} & \textbf{32.0} & +* \\
 & 305 & 200000 & 1898.7 & 32.8 & \textbf{1095.8} & \textbf{25.7} & +* & 1889.7 & 33.1 & \textbf{1093.0} & \textbf{27.7} & +* \\
 & 305 & 500000 & 1380.8 & 23.7 & \textbf{621.4} & \textbf{15.3} & +* & 1372.7 & 26.3 & \textbf{653.9} & \textbf{19.1} & +* \\
 & 305 & 1000000 & 991.5 & 15.8 & \textbf{321.1} & \textbf{7.9} & +* & 984.9 & 16.6 & \textbf{367.7} & \textbf{11.8} & +* \\
 & 610 & 100000 & 2577.2 & 55.9 & \textbf{1355.8} & \textbf{28.8} & +* & 2576.0 & 49.9 & \textbf{1360.6} & \textbf{30.2} & +* \\
 & 610 & 200000 & 2067.3 & 43.1 & \textbf{875.6} & \textbf{17.2} & +* & 2070.3 & 46.4 & \textbf{900.9} & \textbf{17.9} & +* \\
 & 610 & 500000 & 1358.2 & 27.9 & \textbf{412.3} & \textbf{8.4} & +* & 1364.4 & 34.6 & \textbf{439.4} & \textbf{9.6} & +* \\
 & 610 & 1000000 & 876.1 & 17.1 & \textbf{211.1} & \textbf{4.2} & +* & 903.4 & 21.4 & \textbf{225.8} & \textbf{5.1} & +* \\
\hline
\multirow{20}{*}{ca-AstroPh}
 & 1 & 100000 & \textbf{144.5} & \textbf{56.5} & 687.6 & 96.1 & -* & \textbf{125.4} & \textbf{40.3} & 448.6 & 70.7 & -* \\
 & 1 & 200000 & \textbf{79.7} & \textbf{32.4} & 606.5 & 92.7 & -* & \textbf{104.7} & \textbf{32.8} & 366.8 & 69.5 & -* \\
 & 1 & 500000 & \textbf{31.5} & \textbf{12.9} & 478.9 & 75.9 & -* & \textbf{71.2} & \textbf{23.7} & 386.3 & 63.4 & -* \\
 & 1 & 1000000 & \textbf{15.7} & \textbf{6.4} & 364.7 & 61.8 & -* & \textbf{55.8} & \textbf{14.5} & 323.8 & 57.9 & -* \\
 & 14 & 100000 & \textbf{299.6} & \textbf{70.4} & 943.9 & 91.1 & -* & \textbf{294.4} & \textbf{60.3} & 520.6 & 57.8 & -* \\
 & 14 & 200000 & \textbf{163.3} & \textbf{47.6} & 869.4 & 76.8 & -* & \textbf{180.6} & \textbf{35.9} & 470.2 & 64.3 & -* \\
 & 14 & 500000 & \textbf{64.9} & \textbf{20.3} & 682.2 & 64.3 & -* & \textbf{102.5} & \textbf{22.3} & 406.2 & 61.6 & -* \\
 & 14 & 1000000 & \textbf{32.3} & \textbf{10.1} & 466.6 & 55.5 & -* & \textbf{63.9} & \textbf{14.2} & 328.1 & 58.1 & -* \\
 & 133 & 100000 & 1784.0 & 51.5 & \textbf{1518.6} & \textbf{88.1} & +* & 1788.3 & 51.5 & \textbf{1451.8} & \textbf{64.9} & +* \\
 & 133 & 200000 & 1371.2 & 47.6 & \textbf{1266.2} & \textbf{69.0} & +* & 1361.9 & 40.6 & \textbf{1164.7} & \textbf{57.1} & +* \\
 & 133 & 500000 & \textbf{859.5} & \textbf{37.3} & 910.8 & 50.5 & -* & 785.5 & 33.6 & 797.3 & 45.0 & = \\
 & 133 & 1000000 & \textbf{479.8} & \textbf{27.4} & 586.0 & 38.5 & -* & \textbf{406.0} & \textbf{21.2} & 518.6 & 39.3 & -* \\
 & 895 & 100000 & 6051.5 & 104.0 & \textbf{3270.9} & \textbf{49.5} & +* & 6014.3 & 91.2 & \textbf{3226.3} & \textbf{48.8} & +* \\
 & 895 & 200000 & 4895.6 & 92.4 & \textbf{2336.3} & \textbf{41.5} & +* & 4826.2 & 88.5 & \textbf{2303.5} & \textbf{34.7} & +* \\
 & 895 & 500000 & 3290.5 & 50.5 & \textbf{1403.7} & \textbf{36.7} & +* & 3216.5 & 53.3 & \textbf{1383.7} & \textbf{32.2} & +* \\
 & 895 & 1000000 & 2308.4 & 31.4 & \textbf{826.6} & \textbf{30.9} & +* & 2265.1 & 33.4 & \textbf{799.3} & \textbf{22.2} & +* \\
 & 1790 & 100000 & 8050.3 & 79.5 & \textbf{5182.0} & \textbf{62.7} & +* & 8043.9 & 96.3 & \textbf{5189.0} & \textbf{66.6} & +* \\
 & 1790 & 200000 & 6887.7 & 74.6 & \textbf{3887.8} & \textbf{58.4} & +* & 6869.4 & 85.2 & \textbf{3873.7} & \textbf{50.2} & +* \\
 & 1790 & 500000 & 5205.8 & 66.8 & \textbf{2270.8} & \textbf{43.3} & +* & 5179.8 & 73.6 & \textbf{2205.6} & \textbf{29.0} & +* \\
 & 1790 & 1000000 & 3912.6 & 54.5 & \textbf{1274.5} & \textbf{23.6} & +* & 3864.1 & 62.6 & \textbf{1223.0} & \textbf{17.1} & +* \\
\hline
\multirow{20}{*}{ca-CondMat}
 & 1 & 100000 & \textbf{57.0} & \textbf{30.9} & 475.1 & 60.8 & -* & \textbf{56.0} & \textbf{43.7} & 272.4 & 41.0 & -* \\
 & 1 & 200000 & \textbf{40.2} & \textbf{23.7} & 453.0 & 67.9 & -* & \textbf{41.1} & \textbf{31.3} & 338.2 & 52.3 & -* \\
 & 1 & 500000 & \textbf{15.9} & \textbf{9.3} & 374.8 & 63.7 & -* & \textbf{26.1} & \textbf{17.3} & 327.9 & 53.5 & -* \\
 & 1 & 1000000 & \textbf{7.9} & \textbf{4.6} & 292.8 & 63.9 & -* & \textbf{14.3} & \textbf{9.1} & 271.3 & 44.5 & -* \\
 & 14 & 100000 & \textbf{270.6} & \textbf{62.5} & 868.6 & 82.8 & -* & \textbf{258.0} & \textbf{50.8} & 599.4 & 69.2 & -* \\
 & 14 & 200000 & \textbf{151.5} & \textbf{47.7} & 842.8 & 71.9 & -* & \textbf{143.4} & \textbf{40.2} & 470.8 & 57.0 & -* \\
 & 14 & 500000 & \textbf{60.1} & \textbf{19.7} & 697.5 & 68.5 & -* & \textbf{57.0} & \textbf{16.3} & 385.0 & 56.2 & -* \\
 & 14 & 1000000 & \textbf{29.9} & \textbf{9.8} & 484.1 & 62.1 & -* & \textbf{32.3} & \textbf{10.1} & 308.7 & 50.7 & -* \\
 & 146 & 100000 & 1946.4 & 58.0 & \textbf{1622.5} & \textbf{56.3} & +* & 1920.8 & 70.1 & \textbf{1559.7} & \textbf{56.9} & +* \\
 & 146 & 200000 & 1534.6 & 54.3 & \textbf{1351.9} & \textbf{53.4} & +* & 1486.7 & 58.3 & \textbf{1264.4} & \textbf{51.9} & +* \\
 & 146 & 500000 & \textbf{940.6} & \textbf{43.7} & 978.0 & 48.8 & -* & 856.2 & 37.8 & 865.8 & 37.7 & = \\
 & 146 & 1000000 & \textbf{498.7} & \textbf{32.1} & 647.3 & 41.4 & -* & \textbf{437.6} & \textbf{21.5} & 564.2 & 31.1 & -* \\
 & 1068 & 100000 & 8412.2 & 86.3 & \textbf{5031.4} & \textbf{70.6} & +* & 8461.3 & 82.7 & \textbf{5054.2} & \textbf{87.1} & +* \\
 & 1068 & 200000 & 7076.5 & 82.6 & \textbf{3600.7} & \textbf{43.8} & +* & 7150.0 & 84.5 & \textbf{3612.3} & \textbf{50.2} & +* \\
 & 1068 & 500000 & 5006.5 & 78.1 & \textbf{2173.1} & \textbf{37.4} & +* & 5091.4 & 79.7 & \textbf{2163.5} & \textbf{41.3} & +* \\
 & 1068 & 1000000 & 3577.6 & 49.5 & \textbf{1282.9} & \textbf{37.9} & +* & 3616.9 & 40.1 & \textbf{1261.0} & \textbf{35.8} & +* \\
 & 2136 & 100000 & 11636.3 & 123.3 & \textbf{8220.4} & \textbf{80.2} & +* & 11651.0 & 95.4 & \textbf{8150.9} & \textbf{91.7} & +* \\
 & 2136 & 200000 & 10327.3 & 110.5 & \textbf{6378.5} & \textbf{85.7} & +* & 10319.9 & 96.6 & \textbf{6301.0} & \textbf{77.8} & +* \\
 & 2136 & 500000 & 8222.5 & 97.0 & \textbf{3877.4} & \textbf{72.5} & +* & 8178.0 & 91.9 & \textbf{3753.5} & \textbf{51.9} & +* \\
 & 2136 & 1000000 & 6398.5 & 85.0 & \textbf{2187.4} & \textbf{34.3} & +* & 6327.2 & 86.8 & \textbf{2106.7} & \textbf{20.1} & +* \\
\hline
\multirow{20}{*}{ca-GrQc}
 & 1 & 100000 & \textbf{6.0} & \textbf{2.4} & 68.9 & 10.8 & -* & \textbf{8.4} & \textbf{2.9} & 42.0 & 10.7 & -* \\
 & 1 & 200000 & \textbf{6.2} & \textbf{1.8} & 63.5 & 9.1 & -* & \textbf{6.5} & \textbf{2.0} & 38.2 & 9.6 & -* \\
 & 1 & 500000 & \textbf{7.4} & \textbf{2.0} & 44.0 & 7.6 & -* & \textbf{7.0} & \textbf{1.3} & 30.1 & 8.3 & -* \\
 & 1 & 1000000 & \textbf{6.1} & \textbf{1.1} & 28.7 & 6.4 & -* & \textbf{6.6} & \textbf{0.9} & 20.4 & 7.2 & -* \\
 & 12 & 100000 & \textbf{51.4} & \textbf{9.9} & 155.8 & 16.9 & -* & \textbf{53.9} & \textbf{8.2} & 129.5 & 13.6 & -* \\
 & 12 & 200000 & \textbf{34.4} & \textbf{7.4} & 145.9 & 14.4 & -* & \textbf{37.4} & \textbf{6.9} & 74.9 & 9.1 & -* \\
 & 12 & 500000 & \textbf{20.9} & \textbf{5.4} & 99.7 & 9.5 & -* & \textbf{19.8} & \textbf{4.7} & 47.3 & 7.5 & -* \\
 & 12 & 1000000 & \textbf{14.1} & \textbf{4.0} & 59.3 & 6.3 & -* & \textbf{11.7} & \textbf{2.3} & 31.5 & 6.8 & -* \\
 & 64 & 100000 & \textbf{321.8} & \textbf{12.1} & 341.3 & 17.7 & -* & 309.7 & 11.4 & \textbf{303.1} & \textbf{12.0} & +* \\
 & 64 & 200000 & \textbf{237.4} & \textbf{8.9} & 274.3 & 13.9 & -* & \textbf{219.4} & \textbf{8.2} & 225.0 & 9.0 & -* \\
 & 64 & 500000 & \textbf{143.8} & \textbf{7.3} & 172.2 & 8.9 & -* & \textbf{116.7} & \textbf{5.7} & 126.1 & 6.8 & -* \\
 & 64 & 1000000 & \textbf{87.5} & \textbf{8.2} & 94.2 & 5.0 & -* & \textbf{64.1} & \textbf{4.3} & 71.5 & 6.6 & -* \\
 & 207 & 100000 & 784.7 & 18.1 & \textbf{544.3} & \textbf{11.8} & +* & 782.6 & 14.0 & \textbf{530.6} & \textbf{12.2} & +* \\
 & 207 & 200000 & 636.3 & 13.7 & \textbf{414.5} & \textbf{10.8} & +* & 630.3 & 11.7 & \textbf{389.5} & \textbf{8.8} & +* \\
 & 207 & 500000 & 440.5 & 9.8 & \textbf{240.8} & \textbf{9.1} & +* & 425.2 & 8.9 & \textbf{217.0} & \textbf{5.7} & +* \\
 & 207 & 1000000 & 304.7 & 7.4 & \textbf{130.8} & \textbf{6.4} & +* & 278.8 & 6.5 & \textbf{116.9} & \textbf{4.8} & +* \\
 & 415 & 100000 & 1250.2 & 29.9 & \textbf{689.8} & \textbf{10.5} & +* & 1245.7 & 17.6 & \textbf{688.3} & \textbf{10.9} & +* \\
 & 415 & 200000 & 984.3 & 18.1 & \textbf{494.3} & \textbf{9.7} & +* & 981.4 & 11.1 & \textbf{490.3} & \textbf{8.6} & +* \\
 & 415 & 500000 & 688.7 & 11.8 & \textbf{270.9} & \textbf{7.6} & +* & 687.1 & 8.3 & \textbf{266.7} & \textbf{6.8} & +* \\
 & 415 & 1000000 & 494.4 & 8.2 & \textbf{144.2} & \textbf{5.6} & +* & 489.4 & 7.7 & \textbf{141.8} & \textbf{4.7} & +* \\
\hline
\multirow{20}{*}{ca-HepPh}
 & 1 & 100000 & \textbf{74.6} & \textbf{23.6} & 360.0 & 43.7 & -* & \textbf{35.9} & \textbf{11.8} & 162.8 & 45.6 & -* \\
 & 1 & 200000 & \textbf{42.5} & \textbf{17.2} & 320.3 & 39.5 & -* & \textbf{29.0} & \textbf{8.6} & 206.9 & 36.8 & -* \\
 & 1 & 500000 & \textbf{17.5} & \textbf{7.9} & 249.9 & 27.4 & -* & \textbf{22.7} & \textbf{4.6} & 198.5 & 33.1 & -* \\
 & 1 & 1000000 & \textbf{8.7} & \textbf{3.9} & 189.3 & 23.9 & -* & \textbf{19.4} & \textbf{3.8} & 157.1 & 32.4 & -* \\
 & 13 & 100000 & \textbf{162.0} & \textbf{35.5} & 501.7 & 43.7 & -* & \textbf{167.5} & \textbf{24.3} & 379.3 & 34.1 & -* \\
 & 13 & 200000 & \textbf{103.0} & \textbf{34.3} & 477.7 & 40.0 & -* & \textbf{96.4} & \textbf{19.2} & 284.3 & 28.8 & -* \\
 & 13 & 500000 & \textbf{42.4} & \textbf{15.2} & 398.8 & 36.9 & -* & \textbf{42.4} & \textbf{10.1} & 198.6 & 26.1 & -* \\
 & 13 & 1000000 & \textbf{21.1} & \textbf{7.6} & 274.3 & 28.5 & -* & \textbf{28.3} & \textbf{5.3} & 157.2 & 26.8 & -* \\
 & 105 & 100000 & 985.3 & 37.2 & \textbf{897.7} & \textbf{28.9} & +* & 960.4 & 36.0 & \textbf{857.7} & \textbf{40.3} & +* \\
 & 105 & 200000 & 758.5 & 30.8 & 760.1 & 23.1 & = & 737.0 & 29.0 & \textbf{709.9} & \textbf{36.5} & +* \\
 & 105 & 500000 & \textbf{451.9} & \textbf{27.4} & 561.4 & 20.9 & -* & \textbf{432.6} & \textbf{25.1} & 497.4 & 26.3 & -* \\
 & 105 & 1000000 & \textbf{242.1} & \textbf{19.1} & 360.7 & 16.8 & -* & \textbf{225.6} & \textbf{16.0} & 316.3 & 21.0 & -* \\
 & 560 & 100000 & 3148.0 & 55.6 & \textbf{1774.6} & \textbf{35.7} & +* & 3114.3 & 65.6 & \textbf{1774.5} & \textbf{27.4} & +* \\
 & 560 & 200000 & 2465.7 & 39.0 & \textbf{1367.6} & \textbf{32.2} & +* & 2449.2 & 37.2 & \textbf{1362.2} & \textbf{28.4} & +* \\
 & 560 & 500000 & 1735.2 & 23.9 & \textbf{890.6} & \textbf{25.0} & +* & 1729.6 & 24.0 & \textbf{870.7} & \textbf{21.5} & +* \\
 & 560 & 1000000 & 1268.8 & 17.9 & \textbf{536.6} & \textbf{20.6} & +* & 1264.0 & 18.8 & \textbf{510.2} & \textbf{15.6} & +* \\
 & 1120 & 100000 & 4799.2 & 50.5 & \textbf{2910.3} & \textbf{34.7} & +* & 4760.7 & 59.0 & \textbf{2906.4} & \textbf{41.5} & +* \\
 & 1120 & 200000 & 4043.1 & 46.7 & \textbf{2085.6} & \textbf{21.0} & +* & 4014.7 & 56.9 & \textbf{2059.1} & \textbf{25.8} & +* \\
 & 1120 & 500000 & 2932.8 & 35.3 & \textbf{1222.7} & \textbf{15.9} & +* & 2891.9 & 49.7 & \textbf{1198.6} & \textbf{17.5} & +* \\
 & 1120 & 1000000 & 2096.5 & 28.8 & \textbf{703.2} & \textbf{12.4} & +* & 2050.8 & 28.4 & \textbf{679.3} & \textbf{14.5} & +* \\
\hline
\end{tabular}
\end{adjustbox}
\end{table}

\bibliographystyle{unsrt}
\bibliography{final-ref}

@inproceedings{DBLP:conf/ppsn/NeumannR24,
  author       = {Frank Neumann and
                  G{\"{u}}nter Rudolph},
  title        = {Archive-Based Single-Objective Evolutionary Algorithms for Submodular
                  Optimization},
  booktitle    = {{PPSN} {(3)}},
  series       = {Lecture Notes in Computer Science},
  volume       = {15150},
  pages        = {166--180},
  publisher    = {Springer},
  year         = {2024},
  doi          = {10.1007/978-3-031-70071-2_11},
}

@article{DBLP:journals/ai/RoostapourNNF22,
  author       = {Vahid Roostapour and
                  Aneta Neumann and
                  Frank Neumann and
                  Tobias Friedrich},
  title        = {Pareto optimization for subset selection with dynamic cost constraints},
  journal      = {Artif. Intell.},
  volume       = {302},
  pages        = {103597},
  year         = {2022},
  url          = {https://doi.org/10.1016/j.artint.2021.103597},
  doi          = {10.1016/j.artint.2021.103597},
}

@inproceedings{DBLP:conf/ppsn/YanNN24,
  author       = {Xiankun Yan and
                  Aneta Neumann and
                  Frank Neumann},
  title        = {Sliding Window Bi-objective Evolutionary Algorithms for Optimizing
                  Chance-Constrained Monotone Submodular Functions},
  booktitle    = {{PPSN} {(1)}},
  series       = {Lecture Notes in Computer Science},
  volume       = {15148},
  pages        = {20--35},
  publisher    = {Springer},
  year         = {2024},
  url          = {https://doi.org/10.1007/978-3-031-70055-2_2},
  doi          = {10.1007/978-3-031-70055-2_2},
}

@article{DBLP:journals/ipl/KhullerMN99,
  author       = {Samir Khuller and
                  Anna Moss and
                  Joseph Naor},
  title        = {The Budgeted Maximum Coverage Problem},
  journal      = {Inf. Process. Lett.},
  volume       = {70},
  number       = {1},
  pages        = {39--45},
  year         = {1999},
  doi          = {10.1016/S0020-0190(99)00031-9},
  url          = {https://doi.org/10.1016/S0020-0190(99)00031-9},
}

@article{DBLP:journals/jacm/GoemansW95,
  author       = {Michel X. Goemans and
                  David P. Williamson},
  title        = {Improved Approximation Algorithms for Maximum Cut and Satisfiability
                  Problems Using Semidefinite Programming},
  journal      = {J. {ACM}},
  volume       = {42},
  number       = {6},
  pages        = {1115--1145},
  year         = {1995},
  url          = {https://doi.org/10.1145/227683.22768},
  doi          = {10.1145/227683.22768},
}

@incollection{DBLP:books/cu/p/0001G14,
  author       = {Andreas Krause and
                  Daniel Golovin},
  title        = {Submodular Function Maximization},
  booktitle    = {Tractability},
  pages        = {71--104},
  publisher    = {Cambridge University Press},
  year         = {2014},
  url          = {https://doi.org/10.1017/CBO9781139177801.004},
  doi          = {10.1017/CBO9781139177801.004},
}

@article{DBLP:journals/mor/NemhauserW78,
  author       = {George L. Nemhauser and
                  Laurence A. Wolsey},
  title        = {Best Algorithms for Approximating the Maximum of a Submodular Set
                  Function},
  journal      = {Math. Oper. Res.},
  volume       = {3},
  number       = {3},
  pages        = {177--188},
  year         = {1978},
  url          = {https://doi.org/10.1287/moor.3.3.177},
  doi          = {10.1287/moor.3.3.177},
}

@inproceedings{DBLP:conf/soict/Ong15,
  author       = {Yew{-}Soon Ong},
  title        = {Towards Evolutionary Multitasking: {A} New Paradigm},
  booktitle    = {SoICT},
  pages        = {2},
  publisher    = {{ACM}},
  year         = {2015},
}

@article{Tan2023,
  title        = {Pareto optimization with small data by learning across common objective spaces},
  volume       = {13},
  pages        = {7842},
  journal      = {Scientific Reports},
  publisher    = {Springer Science and Business Media LLC},
  author       = {Tan, Chin Sheng and Gupta, Abhishek and Ong, Yew-Soon and Pratama, Mahardhika and Tan, Puay Siew and Lam, Siew Kei},
  year         = {2023},
  url          = {https://doi.org/10.1038/s41598-023-33414-6},
  doi          = {10.1038/s41598-023-33414-6},
}

@inproceedings{DBLP:conf/cec/DaGOF16,
  author       = {Bingshui Da and
                  Abhishek Gupta and
                  Yew{-}Soon Ong and
                  Liang Feng},
  title        = {Evolutionary multitasking across single and multi-objective formulations
                  for improved problem solving},
  booktitle    = {{CEC}},
  pages        = {1695--1701},
  publisher    = {{IEEE}},
  year         = {2016},
  doi          = {10.1109/CEC.2016.7743992},
  url          = {https://doi.org/10.1109/CEC.2016.7743992},
}

@inproceedings{DBLP:conf/aaai/RossiA15,
  author       = {Ryan A. Rossi and
                  Nesreen K. Ahmed},
  title        = {The Network Data Repository with Interactive Graph Analytics and Visualization},
  booktitle    = {{AAAI}},
  pages        = {4292--4293},
  publisher    = {{AAAI} Press},
  year         = {2015},
  doi          = {10.5555/2888116.2888372},
  url          = {https://doi.org/10.5555/2888116.2888372},
}

@inproceedings{DBLP:conf/ijcai/0001Q0021,
  author       = {Chao Bian and
                  Chao Qian and
                  Frank Neumann and
                  Yang Yu},
  title        = {Fast Pareto Optimization for Subset Selection with Dynamic Cost Constraints},
  booktitle    = {{IJCAI}},
  pages        = {2191--2197},
  publisher    = {ijcai},
  year         = {2021},
  url          = {https://doi.org/10.24963/ijcai.2021/302},
  doi          = {10.24963/ijcai.2021/302},
}

@inproceedings{DBLP:conf/gecco/Lengler0N25,
  author       = {Johannes Lengler and
                  Aneta Neumann and
                  Frank Neumann},
  title        = {Runtime Analysis of Evolutionary Multitasking for Classical Benchmark
                  Problems},
  booktitle    = {{GECCO}},
  pages        = {1613--1621},
  publisher    = {{ACM}},
  year         = {2025},
  url          = {https://doi.org/10.1145/3712256.3726369},
  doi          = {10.1145/3712256.3726369},
}

@article{DBLP:journals/tec/WangXJT25,
  author       = {Zhenzhong Wang and
                  Dejun Xu and
                  Min Jiang and
                  Kay Chen Tan},
  title        = {Spatial-Temporal Knowledge Transfer for Dynamic Constrained Multiobjective
                  Optimization},
  journal      = {{IEEE} Trans. Evol. Comput.},
  volume       = {29},
  number       = {5},
  pages        = {1990--2003},
  year         = {2025},
  doi          = {10.1109/TEVC.2024.3449142},
  url          = {https://doi.org/10.1109/TEVC.2024.3449142},
}

@inproceedings{DBLP:conf/ppsn/DoN20,
  author       = {Anh Viet Do and
                  Frank Neumann},
  title        = {Maximizing Submodular or Monotone Functions Under Partition Matroid
                  Constraints by Multi-objective Evolutionary Algorithms},
  booktitle    = {{PPSN} {(2)}},
  series       = {Lecture Notes in Computer Science},
  volume       = {12270},
  pages        = {588--603},
  publisher    = {Springer},
  year         = {2020},
  url          = {https://doi.org/10.1007/978-3-030-58115-2_41},
  doi          = {10.1007/978-3-030-58115-2_41},
}

@inproceedings{DBLP:conf/stoc/ChenP22,
  author       = {Xi Chen and
                  Binghui Peng},
  title        = {On the complexity of dynamic submodular maximization},
  booktitle    = {{STOC}},
  pages        = {1685--1698},
  publisher    = {{ACM}},
  year         = {2022},
  url          = {https://doi.org/10.1145/3519935.3519951},
  doi          = {10.1145/3519935.3519951},
}

@article{DBLP:journals/tec/QiaoYQ0SYLT23,
  author       = {Kangjia Qiao and
                  Kunjie Yu and
                  Boyang Qu and
                  Jing Liang and
                  Hui Song and
                  Caitong Yue and
                  Hongyu Lin and
                  Kay Chen Tan},
  title        = {Dynamic Auxiliary Task-Based Evolutionary Multitasking for Constrained
                  Multiobjective Optimization},
  journal      = {{IEEE} Trans. Evol. Comput.},
  volume       = {27},
  number       = {3},
  pages        = {642--656},
  year         = {2023},
  doi          = {10.1109/TEVC.2022.3175065},
  url          = {https://doi.org/10.1109/TEVC.2022.3175065},
}

@article{DBLP:journals/cogcom/OsabaSMH22,
  author       = {Eneko Osaba and
                  Javier Del Ser and
                  Aritz D. Martinez and
                  Amir Hussain},
  title        = {Evolutionary Multitask Optimization: a Methodological Overview, Challenges,
                  and Future Research Directions},
  journal      = {Cogn. Comput.},
  volume       = {14},
  number       = {3},
  pages        = {927--954},
  year         = {2022},
  url          = {https://doi.org/10.1007/s12559-022-10012-8},
  doi          = {10.1007/s12559-022-10012-8},
}

@inproceedings{DBLP:conf/gecco/Witt14a,
  author       = {Carsten Witt},
  title        = {Bioinspired computation in combinatorial optimization: algorithms
                  and their computational complexity},
  booktitle    = {{GECCO} (Companion)},
  pages        = {647--686},
  publisher    = {{ACM}},
  year         = {2014},
  url          = {https://doi.org/10.1145/2464576.2466738},
  doi          = {10.1145/2464576.2466738},
}

@book{DBLP:series/ncs/2020DN,
  editor       = {Benjamin Doerr and
                  Frank Neumann},
  title        = {Theory of Evolutionary Computation - Recent Developments in Discrete
                  Optimization},
  series       = {Natural Computing Series},
  publisher    = {Springer},
  year         = {2020},
  url          = {https://doi.org/10.1007/978-3-030-29414-4},
  doi          = {10.1007/978-3-030-29414-4},
}

@inproceedings{DBLP:conf/foga/ScottJ23,
  author       = {Eric O. Scott and
                  Kenneth A. {De Jong}},
  title        = {First Complexity Results for Evolutionary Knowledge Transfer},
  booktitle    = {{FOGA}},
  pages        = {140--151},
  publisher    = {{ACM}},
  year         = {2023},
  url          = {https://doi.org/10.1145/3594805.3607137},
  doi          = {10.1145/3594805.3607137},
}

@inproceedings{DBLP:conf/cec/ScottJ24,
  author       = {Eric O. Scott and
                  Kenneth A. {De Jong}},
  title        = {Varying Difficulty of Knowledge Reuse in Benchmarks for Evolutionary
                  Knowledge Transfer},
  booktitle    = {{CEC}},
  pages        = {1--8},
  publisher    = {{IEEE}},
  year         = {2024},
  doi          = {10.1109/CEC60901.2024.10612149},
  url          = {https://doi.org/10.1109/CEC60901.2024.10612149},
}

@inproceedings{DBLP:conf/gecco/Don0025,
  author       = {Thilina Pathirage Don and
                  Aneta Neumann and
                  Frank Neumann},
  title        = {Evolutionary Multitasking for the Scenario-based Travelling Thief
                  Problem},
  booktitle    = {{GECCO}},
  pages        = {809--817},
  publisher    = {{ACM}},
  year         = {2025},
  url          = {https://doi.org/10.1145/3712256.3726468},
  doi          = {10.1145/3712256.3726468},
}

@inproceedings{7848632,
  author       = {Yuan, Yuan and Ong, Yew-Soon and Gupta, Abhishek and Tan, Puay Siew and Xu, Hua},

    title        = {Evolutionary multitasking in permutation-based combinatorial optimization problems: Realization with {TSP}, {QAP}, {LOP}, and {JSP}},
  booktitle    = {2016 {IEEE} Region 10 Conference ({TENCON})},
  year         = {2016},
  pages        = {3157--3164},
  doi          = {10.1109/TENCON.2016.7848632},
}

@article{DBLP:journals/isci/HuLSM22,
  author       = {Ziyu Hu and
                  Yulin Li and
                  Hao Sun and
                  Xuemin Ma},
  title        = {Multitasking multiobjective optimization based on transfer component
                  analysis},
  journal      = {Inf. Sci.},
  volume       = {605},
  pages        = {182--201},
  year         = {2022},
  url          = {https://doi.org/10.1016/j.ins.2022.05.037},
  doi          = {10.1016/j.ins.2022.05.037},
}

@article{SUN2023504,
  title        = {Multi-objective evolutionary multitasking algorithm based on cross-task transfer solution matching strategy},
  journal      = {ISA Transactions},
  volume       = {138},
  pages        = {504--520},
  year         = {2023},
  issn         = {0019-0578},
  url          = {https://doi.org/10.1016/j.isatra.2023.03.015},
  doi          = {10.1016/j.isatra.2023.03.015},
  author       = {Hao Sun and Pengfei Chen and Ziyu Hu and Lixin Wei},
}

@article{DBLP:journals/tcyb/LinWMGLC24,
  author       = {Qiuzhen Lin and
                  Zhongjian Wu and
                  Lijia Ma and
                  Maoguo Gong and
                  Jianqiang Li and
                  Carlos A. Coello Coello},
  title        = {Multiobjective Multitasking Optimization With Decomposition-Based
                  Transfer Selection},
  journal      = {{IEEE} Trans. Cybern.},
  volume       = {54},
  number       = {5},
  pages        = {3146--3159},
  year         = {2024},
  url          = {https://doi.org/10.1109/TCYB.2023.3266241},
  doi          = {10.1109/TCYB.2023.3266241},
}

@article{DBLP:journals/tcyb/LinLTG21,
  author       = {Jiabin Lin and
                  Hai{-}Lin Liu and
                  Kay Chen Tan and
                  Fangqing Gu},
  title        = {An Effective Knowledge Transfer Approach for Multiobjective Multitasking
                  Optimization},
  journal      = {{IEEE} Trans. Cybern.},
  volume       = {51},
  number       = {6},
  pages        = {3238--3248},
  year         = {2021},
  url          = {https://doi.org/10.1109/TCYB.2020.2969025},
  doi          = {10.1109/TCYB.2020.2969025},
}

@misc{DBLP:conf/gecco/ANNON,
  title        = {Analysis of Multitasking Pareto Optimization for Monotone Submodular Problems},
  author       = {Liam Wigney and Frank Neumann},
  year         = {2026},
  eprint       = {2604.15068},
  archivePrefix = {arXiv},
  primaryClass = {cs.NE},
  url          = {https://arxiv.org/abs/2604.15068},
}

@inproceedings{DBLP:conf/gecco/WigneyNON25,
  author       = {Liam Wigney and
                  Aneta Neumann and
                  Yew{-}Soon Ong and
                  Frank Neumann},
  title        = {On the Use of Matching Algorithms to Transfer Solutions for the Travelling
                  Salesperson Problem},
  booktitle    = {{GECCO}},
  pages        = {845--853},
  publisher    = {{ACM}},
  year         = {2025},
  doi          = {10.1145/3712256.3726469},
}

@article{NeumannNeumannTCS23,
  author  = {Aneta Neumann and Frank Neumann},
  title   = {Optimizing Monotone Chance-Constrained Submodular Functions Using
             Evolutionary Multiobjective Algorithms},
  journal = {Evol. Comput.},
  volume  = {33},
  number  = {3},
  pages   = {363--393},
  year    = {2025},
  doi     = {10.1162/evco_a_00360},
}

@article{DBLP:journals/tec/ZhangMNZT21,
  author       = {Fangfang Zhang and
                  Yi Mei and
                  Su Nguyen and
                  Mengjie Zhang and
                  Kay Chen Tan},
  title        = {Surrogate-Assisted Evolutionary Multitask Genetic Programming for
                  Dynamic Flexible Job Shop Scheduling},
  journal      = {{IEEE} Trans. Evol. Comput.},
  volume       = {25},
  number       = {4},
  pages        = {651--665},
  year         = {2021}
}

@article{DBLP:journals/ec/FriedrichN15,
  author       = {Tobias Friedrich and
                  Frank Neumann},
  title        = {Maximizing Submodular Functions under Matroid Constraints by Evolutionary
                  Algorithms},
  journal      = {Evol. Comput.},
  volume       = {23},
  number       = {4},
  pages        = {543--558},
  year         = {2015}
}

\end{document}